\documentclass{article}

\usepackage[colorlinks=true, linkcolor=blue, citecolor=blue,urlcolor=black]{hyperref}

     \usepackage[nonatbib, preprint]{neurips_2021}

\usepackage[utf8]{inputenc} 
\usepackage[T1]{fontenc}    
\usepackage{hyperref}       
\usepackage{url}            
\usepackage{booktabs}       
\usepackage{amsfonts}       
\usepackage{nicefrac}       
\usepackage{microtype}      
\usepackage{mathtools}
\usepackage{natbib}
\usepackage{xcolor}
\usepackage{color}

\usepackage[ruled]{algorithm2e}
\usepackage{algorithmic}

\usepackage{amsmath}
\usepackage{mathtools}
\usepackage{amsthm}
\usepackage{amssymb}

\usepackage{chngcntr}
\usepackage{apptools}
\AtAppendix{\counterwithin{lemma}{section}}

\newcommand{\calA}{{\mathcal{A}}}

\newcommand{\calS}{{\mathcal{S}}}
\newcommand{\calF}{{\mathcal{F}}}

\newcommand{\calT}{{\mathcal{T}}}

\newcommand{\calM}{{\mathcal{M}}}

\newcommand{\one}{\boldsymbol{1}}

\DeclarePairedDelimiter\abs{\lvert}{\rvert}

\newcommand{\eat}[1]{}

\newcommand{\rbr}[1]{\left(#1\right)}
\newcommand{\sbr}[1]{\left[#1\right]}
\newcommand{\cbr}[1]{\left\{#1\right\}}

\newcommand{\abr}[1]{\left|#1\right|}

\newcommand{\T}{\ensuremath{T_\star}}
\newcommand{\B}{B_\star}
\newcommand{\cmin}{\ensuremath{c_{\min}}}
\newcommand{\cmininv}{\ensuremath{c^{-1}_{\min}}}

\newcommand{\ssp}{\texttt{PSRL-SSP}}

\newcommand{\hatthetal}{\widehat{\theta}_\ell}
\newcommand{\sinit}{s_\text{init}}

\newcommand{\sumsp}{\sum_{s'\in\calS^+}}
\newcommand{\sumsdp}{\sum_{s''\in\calS^+}}

\newcommand{\suml}{\sum_{\ell = 1}^{L_M}}
\newcommand{\summ}{\sum_{m = 1}^{M}}
\newcommand{\sumtmplus}{\sum_{t = t_m+1}^{t_{m+1}-1}}
\newcommand{\oneomega}{\one_{\Omega^\ell_{s_t, a_t}}}

\newcommand{\sumtl}{\sum_{t=t_\ell}^{t_{\ell+1}-1}}
\newcommand{\sumtm}{\sum_{t=t_m}^{t_{m+1}-1}}

\newcommand{\thetalm}{\theta_{\ell}}

\newcommand{\Alsa}{A_\ell(s, a)}
\newcommand{\Alstat}{A_\ell(s_t, a_t)}
\newcommand{\Aldef}{\frac{\log (SAn_\ell^+(s, a)/\delta)}{n_\ell^+(s, a)}}
\newcommand{\Vlstat}{\mathbb{V}_\ell(s_t, a_t)}

\newcommand{\field}[1]{\mathbb{#1}}

\newcommand{\E}{\field{E}}

\newtheorem{lemma}{Lemma}
\newtheorem{theorem}{Theorem}

\newtheorem{remark}{Remark}

\newtheorem{assumption}{Assumption}

\newcommand{\order}{\ensuremath{\mathcal{O}}}

\newcommand{\otil}{\ensuremath{\tilde{\mathcal{O}}}}

\usepackage{prettyref}

\newcommand{\savehyperref}[2]{\texorpdfstring{\hyperref[#1]{#2}}{#2}}
\newrefformat{eq}{\savehyperref{#1}{Eq.~\textup{(\ref*{#1})}}}
\newrefformat{eqn}{\savehyperref{#1}{Equation~\ref*{#1}}}
\newrefformat{lem}{\savehyperref{#1}{Lemma~\ref*{#1}}}
\newrefformat{def}{\savehyperref{#1}{Definition~\ref*{#1}}}
\newrefformat{line}{\savehyperref{#1}{Line~\ref*{#1}}}
\newrefformat{thm}{\savehyperref{#1}{Theorem~\ref*{#1}}}
\newrefformat{corr}{\savehyperref{#1}{Corollary~\ref*{#1}}}
\newrefformat{cor}{\savehyperref{#1}{Corollary~\ref*{#1}}}
\newrefformat{sec}{\savehyperref{#1}{Section~\ref*{#1}}}
\newrefformat{subsec}{\savehyperref{#1}{Section~\ref*{#1}}}
\newrefformat{app}{\savehyperref{#1}{Appendix~\ref*{#1}}}
\newrefformat{assum}{\savehyperref{#1}{Assumption~\ref*{#1}}}
\newrefformat{ex}{\savehyperref{#1}{Example~\ref*{#1}}}
\newrefformat{fig}{\savehyperref{#1}{Figure~\ref*{#1}}}
\newrefformat{alg}{\savehyperref{#1}{Algorithm~\ref*{#1}}}
\newrefformat{rem}{\savehyperref{#1}{Remark~\ref*{#1}}}
\newrefformat{conj}{\savehyperref{#1}{Conjecture~\ref*{#1}}}
\newrefformat{prop}{\savehyperref{#1}{Proposition~\ref*{#1}}}
\newrefformat{proto}{\savehyperref{#1}{Protocol~\ref*{#1}}}
\newrefformat{prob}{\savehyperref{#1}{Problem~\ref*{#1}}}
\newrefformat{claim}{\savehyperref{#1}{Claim~\ref*{#1}}}
\newrefformat{que}{\savehyperref{#1}{Question~\ref*{#1}}}
\newrefformat{op}{\savehyperref{#1}{Open Problem~\ref*{#1}}}
\newrefformat{fn}{\savehyperref{#1}{Footnote~\ref*{#1}}}
\newrefformat{tab}{\savehyperref{#1}{Table~\ref*{#1}}}
\newrefformat{fig}{\savehyperref{#1}{Figure~\ref*{#1}}}

\title{Online Learning for Stochastic Shortest Path Model via Posterior Sampling}

\author{
  Mehdi Jafarnia-Jahromi \\
  University of Southern California\\
  \texttt{mjafarni@usc.edu} \\
  \And
  Liyu Chen\\ 
  University of Southern California\\
  \texttt{liyuc@usc.edu} \\
  \And
  Rahul Jain \\
  University of Southern California\\
  \texttt{rahul.jain@usc.edu} \\
  \And
  Haipeng Luo \\
  University of Southern California\\
  \texttt{haipengl@usc.edu} \\
}

\begin{document}
\maketitle

\begin{abstract}
We consider the problem of online reinforcement learning for the Stochastic Shortest Path (SSP) problem modeled as an unknown MDP with an absorbing state. We propose \ssp, a simple posterior sampling-based reinforcement learning algorithm for the SSP problem. The algorithm operates in epochs. At the beginning of each epoch, a sample is drawn from the posterior distribution on the unknown model dynamics, and the optimal policy with respect to the drawn sample is followed during that epoch. An epoch completes if either the  number of visits to the goal state in the current epoch exceeds that of the previous epoch, or the number of visits to any of the state-action pairs is doubled. We establish a Bayesian regret bound of $\otil(\B S\sqrt{AK})$, where $\B$ is an upper bound on the expected cost of the optimal policy, $S$ is the size of the state space, $A$ is the size of the action space, and $K$ is the number of episodes. The algorithm only requires the knowledge of the prior distribution, and has no hyper-parameters to tune. It is the first such posterior sampling algorithm and outperforms numerically previously proposed optimism-based algorithms.
\end{abstract}


\section{Introduction}
\label{sec:intro}
Stochastic Shortest Path (SSP) model considers the problem of an agent interacting with an environment to reach a predefined goal state while minimizing the cumulative expected cost. Unlike the finite-horizon and discounted Markov Decision Processes (MDPs) settings, in the SSP model,  the horizon of interaction between the agent and the environment depends on the agent's actions, and can possibly be unbounded (if the goal is not reached). A wide variety of goal-oriented control and reinforcement learning (RL) problems such as navigation, game playing, etc. can be formulated as SSP problems. In the RL setting,  where the SSP model is unknown, the agent interacts with the environment in $K$ \textit{episodes}. Each episode begins at a predefined initial state and ends when the agent reaches the goal (note that it might never reach the goal). We consider the setting where the state and action spaces are finite, the cost function is known, but the transition kernel is unknown. The performance of the agent is measured through the notion of \textit{regret}, i.e., the difference between the cumulative cost of the learning algorithm and that of the optimal policy during the $K$ episodes. 

The agent has to balance the well-known trade-off between \textit{exploration} and \textit{exploitation}: should the agent \textit{explore} the environment to gain information for future decisions, or should it \textit{exploit} the current information to minimize the cost? A general way to balance the exploration-exploitation trade-off is to use the \textit{Optimism in the Face of Uncertainty} (OFU) principle \citep{lai1985asymptotically}. The idea is to construct a set of plausible models based on the available information, select the model associated with the minimum cost, and follow the optimal policy with respect to the selected model. This idea is widely used in the RL literature for MDPs (e.g., \citep{jaksch2010near,azar2017minimax,fruit2018efficient,jin2018q,wei2020model,wei2021learning}) and also for SSP models \citep{tarbouriech2020no,rosenberg2020near,rosenberg2020stochastic,chen2021finding,tarbouriech2021stochastic}.

An alternative fundamental idea to encourage exploration is to use Posterior Sampling (PS) (also known as Thompson Sampling) \citep{thompson1933likelihood}. The idea is to maintain the posterior distribution on the unknown model parameters based on the available information and the prior distribution. PS algorithms usually proceed in \textit{epochs}. In the beginning of an epoch, a model is sampled from the posterior. The actions during the epoch are then selected according to the optimal policy associated with the sampled model. PS algorithms have two main advantages over OFU-type algorithms. First, the prior knowledge of the environment can be incorporated through the prior distribution. Second, PS algorithms have shown superior numerical performance on multi-armed bandit problems \citep{scott2010modern,chapelle2011empirical}, and MDPs \citep{osband2013more,osband2017posterior,ouyang2017learning}. 

The main difficulty in designing PS algorithms is the design of the epochs. In the basic setting of bandit problems, one can simply sample at every time step \citep{chapelle2011empirical}. In finite-horizon MDPs, where the length of an episode is predetermined and fixed, the epochs and episodes coincide, i.e., the agent can sample from the posterior distribution at the beginning of each episode \citep{osband2013more}. However, in the general SSP model, where the length of each episode is not predetermined and can possibly be unbounded, these natural choices for the epoch do not work. Indeed, the agent needs to switch policies during an episode if the current policy cannot reach the goal.

In this paper, we propose \ssp, the first PS-based RL algorithm for the SSP model. \ssp~starts a new epoch based on two criteria. According to the first criterion, a new epoch starts if the number of episodes within the current epoch exceeds that of the previous epoch. The second criterion is triggered when the number of visits to any state-action pair is doubled during an epoch, similar to the one used by \cite{bartlett2009regal,jaksch2010near,filippi2010optimism,dann2015sample,ouyang2017learning,rosenberg2020near}.
Intuitively speaking, in the early stages of the interaction between the agent and the environment, the second criterion triggers more often. This criterion is responsible for switching policies during an episode if the current policy cannot reach the goal. In the later stages of the interaction, the first criterion triggers more often and  encourages exploration. We prove a Bayesian regret bound of $\otil(\B S\sqrt{AK})$, where $S$ is the number of states, $A$ is the number of actions, $K$ is the number of episodes, and $\B$ is an upper bound on the expected cost of the optimal policy. This is similar to the regret bound of \cite{rosenberg2020near} 
and has a gap of $\sqrt{S}$ with the minimax lower bound. We note that concurrent works of \citet{tarbouriech2021stochastic} and \citet{cohen2021minimax} have closed the gap via OFU algorithms and blackbox reduction to the finite-horizon, respectively. However, the goal of this paper is not to match the minimax regret bound, but rather to introduce the first PS algorithm that has near-optimal regret bound with superior numerical performance than OFU algorithms. This is verified with the experiments in Section~\ref{sec: experiments}. The $\sqrt{S}$ gap with the lower bound exists for the PS algorithms in the finite-horizon \cite{osband2013more} and the infinite-horizon average-cost MDPs \citep{ouyang2017learning} as well. Thus, it remains an open question whether it is possible to achieve the lower bound via PS algorithms in these settings.

\paragraph*{Related Work.}
\textbf{Posterior Sampling.} The idea of PS algorithms dates back to the pioneering work of \cite{thompson1933likelihood}. The algorithm was ignored for several decades until recently. In the past two decades, PS algorithms have successfully been developed for various settings including multi-armed bandits (e.g., \cite{scott2010modern,chapelle2011empirical,kaufmann2012thompson,agrawal2012analysis,agrawal2013thompson}), MDPs (e.g., \citep{strens2000bayesian,osband2013more,fonteneau2013optimistic,gopalan2015thompson,osband2017posterior,kim2017thompson,ouyang2017learning,banjevic2019thompson}), Partially Observable MDPs \citep{jafarnia2021online}, and Linear Quadratic Control (e.g., \citep{abeille2017thompson,ouyang2017learningbased}). The interested reader is referred to \cite{russo2017tutorial} and references therein for a more comprehensive literature review.

\textbf{Online Learning in SSP.} Another related line of work is online learning in the SSP model which was introduced by \cite{tarbouriech2020no}. They proposed an algorithm with $\otil(K^{2/3})$ regret bound. Subsequent work of \cite{rosenberg2020near} improved the regret bound to $\otil(\B S\sqrt{AK})$. The concurrent works of \cite{cohen2021minimax,tarbouriech2021stochastic} proved a minimax regret bound of $\otil(\B\sqrt{SAK})$. However, none of these works propose a PS-type algorithm. We refer the interested reader to \cite{rosenberg2020stochastic,chen2020minimax,chen2021finding} for the SSP model with adversarial costs and \cite{tarbouriech2021sample} for sample complexity of the SSP model with a generative model.

\textbf{Comparison with \cite{ouyang2017learning}.} Our work is  related to \cite{ouyang2017learning} which proposes \texttt{TSDE}, a PS algorithm for infinite-horizon average-cost MDPs. However, clear distinctions exist both in the algorithm and analysis. From the algorithmic perspective, our first criterion in determining the epoch length is different from \texttt{TSDE}. Note that using the same epochs as \texttt{TSDE} leads to a sub-optimal regret bound of $\order(K^{2/3})$ in the SSP model setting. Moreover, following Hoeffding-type concentration as in \texttt{TSDE}, yields a regret bound of $\order(K^{2/3})$ in the SSP model setting. Instead, we propose a different analysis using Bernstein-type concentration inspired by the work of \cite{rosenberg2020near} to achieve the $\order(\sqrt{K})$ regret bound (see Lemma~\ref{lem: r3}).

\section{Preliminaries}
\label{sec: preliminaries}

A Stochastic Shortest Path (SSP) model is denoted by $\calM = (\calS, \calA, c, \theta, \sinit, g)$ where $\calS$ is the state space, $\calA$ is the action space, $c: \calS \times \calA \to [0, 1]$ is the cost function, $\sinit \in \calS$ is the initial state, $g \notin \calS$ is the goal state, and $\theta : \calS^+ \times \calS \times \calA \to [0, 1]$ represents the transition kernel such that $\theta(s' | s, a) = \mathbb{P}(s_t'=s'|s_t=s, a_t=a)$ where $\calS^+ = \calS \cup \{g\}$ includes the goal state as well. Here $s_t \in \calS$ and $a_t \in \calA$ are the state and action at time $t=1, 2, 3, \cdots$ and $s_t' \in \calS^+$ is the subsequent state. We assume that the initial state $\sinit$ is a fixed and known state and $\calS$ and $\calA$ are finite sets with size $S$ and $A$, respectively. A stationary policy is a deterministic map $\pi: \calS \to \calA$ that maps a state to an action. The \textit{value function} (also called the \textit{cost-to-go function}) associated with policy $\pi$ is a function $V^\pi(\cdot;\theta): \calS^+ \to [0, \infty]$ given by $V^{\pi}(g;\theta) = 0$ and $V^{\pi}(s;\theta) := \E[\sum_{t=1}^{\tau_{\pi}(s)}c(s_t, \pi(s_t)) | s_1=s]$ for $s \in \calS$, where $\tau_{\pi}(s)$  is  the number of steps before reaching the goal state (a random variable) if the initial state is $s$ and policy $\pi$ is followed throughout the episode. Here, we use the notation $V^\pi(\cdot;\theta)$ to explicitly show the dependence of the value function on $\theta$. Furthermore, the optimal value function can be defined as $V(s;\theta) = \min_{\pi} V^\pi(s;\theta)$. Policy $\pi$ is called \textit{proper} if the goal state is reached with probability $1$, starting from any initial state and following $\pi$ (i.e., $\max_s \tau_\pi (s) < \infty$ almost surely), otherwise it is called \textit{improper}.

We consider the reinforcement learning problem of an agent interacting with an SSP model $\calM = (\calS, \calA, c, \theta_*, \sinit, g)$ whose transition kernel $\theta_*$ is randomly generated according to the prior distribution $\mu_1$ at the beginning and is then fixed. We will focus on SSP models with transition kernels in the set $\Theta_{\B}$ with the following standard properties:
\begin{assumption}
\label{ass: class of ssp}
For all $\theta \in \Theta_{\B}$, the following holds: (1) there exists a proper policy,
(2) for all improper policies $\pi_\theta$, there exists a state $s \in \calS$, such that $V^{\pi_\theta}(s;\theta) = \infty$, and (3) the optimal value function $V(\cdot;\theta)$ satisfies $\max_s V(s;\theta) \leq \B$.
\end{assumption}
\citet{bertsekas1991analysis} prove that the first two conditions in Assumption~\ref{ass: class of ssp} imply that for each $\theta \in \Theta_{\B}$, the optimal policy is stationary, deterministic, proper, and can be obtained by the minimizer of the \textit{Bellman optimality equations} given by
\begin{align}
\label{eq: Bellman equation}
V(s;\theta) = \min_a \Big\{c(s, a) + \sum_{s'\in \calS^+}\theta(s'|s, a)V(s';\theta)\Big\}, \quad \forall s \in \calS.
\end{align}
Standard techniques such as Value Iteration and Policy Iteration can be used to compute the optimal policy if the SSP model is known \citep{bertsekas2017dynamic}. Here, we assume that $\calS$, $\calA$, and the cost function $c$ are known to the agent, however, the transition kernel $\theta_*$ is unknown. Moreover, we assume that the support of the prior distribution $\mu_1$ is a subset of $\Theta_{\B}$.

The agent interacts with the environment in $K$ episodes. Each episode starts from the initial state $\sinit$ and ends at the goal state $g$ (note that the agent may never reach the goal). At each time $t$, the agent observes state $s_t$ and takes action $a_t$. The environment then yields the next state $s_t' \sim \theta_*(\cdot | s_t, a_t)$. If the goal is reached (i.e., $s_t' = g$), then the current episode completes, a new episode starts, and $s_{t+1} = \sinit$. If the goal is not reached (i.e., $s_t' \neq g$), then $s_{t+1}=s_t'$. The goal of the agent is to minimize the expected cumulative cost after $K$ episodes, or equivalently, minimize the \textit{Bayesian regret} defined as
\begin{align*}
R_K &:= \E\sbr{\sum_{t=1}^{T_K} c(s_t, a_t) - KV(\sinit;\theta_*)},
\end{align*} 
where $T_K$ is the total number of time steps before reaching the goal state for the $K$th time, and $V(\sinit;\theta_*)$ is the optimal value function from \eqref{eq: Bellman equation}. Here, expectation is with respect to the prior distribution $\mu_1$ for $\theta_*$, the horizon $T_K$, the randomness in the state transitions, and the randomness of the algorithm. If the agent does not reach the goal state at any of the episodes (i.e., $T_K = \infty$), we define $R_K = \infty$.


\section{A Posterior Sampling RL Algorithm for SSP Models}
\label{sec: algorithm}
In this section, we propose the Posterior Sampling Reinforcement Learning (\ssp) ~algorithm (Algorithm~\ref{alg: posterior sampling}) for the SSP model. The input of the algorithm is the prior distribution $\mu_1$. At time $t$, the agent maintains the posterior distribution $\mu_t$ on the unknown parameter $\theta_*$ given by $\mu_t(\Theta) = \mathbb{P}(\theta_* \in \Theta | \calF_t)$ for any set $\Theta \subseteq \Theta_{\B}$. Here $\calF_t$ is the information available at time $t$ (i.e., the sigma algebra generated by $s_1, a_1, \cdots, s_{t-1}, a_{t-1}, s_t$). Upon observing state $s_t'$ by taking action $a_t$ at state $s_t$, the posterior can be updated according to 
\begin{align}
\label{eq: update rule}
\mu_{t+1}(d\theta) = \frac{\theta(s_t'|s_t, a_t)\mu_t(d\theta)}{\int \theta'(s_t'|s_t, a_t)\mu_t(d\theta')}.
\end{align}

The \ssp~algorithm  proceeds in epochs $\ell = 1, 2, 3, \cdots$. Let $t_\ell$ denote the start time of epoch $\ell$. In the beginning of epoch $\ell$, parameter $\theta_\ell$ is sampled from the posterior distribution $\mu_{t_\ell}$ and the actions within that epoch are chosen according to the optimal policy with respect to $\theta_\ell$. Each epoch ends if either of the two stopping criteria are satisfied. The first criterion is triggered if the number of visits to the goal state during the current epoch (denoted by $K_\ell$) exceeds that of the previous epoch. This ensures that $K_\ell \leq K_{\ell-1}+1$ for all $\ell$. The second criterion is triggered if the number of visits to  any of the state-action pairs is doubled compared to the beginning of the epoch. This guarantees that $n_t(s, a) \leq 2n_{t_\ell}(s, a)$ for all $(s, a)$ where $n_t(s, a) = \sum_{\tau=1}^{t-1} \one_{\{s_\tau=s, a_\tau=a\}}$ denotes the number of visits to state-action pair $(s, a)$ before time $t$. 

The second stopping criterion is similar to that used by \citet{jaksch2010near,rosenberg2020near}, and is one of the two stopping criteria used in the posterior sampling algorithm (\texttt{TSDE}) for the infinite-horizon average-cost MDPs \citep{ouyang2017learning}. This stopping criterion is crucial since it allows the algorithm to switch policies if the generated policy is improper and cannot reach the goal. We note that updating the policy only at the beginning of an episode (as done in the posterior sampling for finite-horizon MDPs \citep{osband2013more}) does not work for SSP models, because if the generated policy in the beginning of the episode is improper, the goal is never reached and the regret is infinity. 

The first stopping criterion is novel. A similar stopping criterion used in the posterior sampling for infinite-horizon MDPs \citep{ouyang2017learning} is based on the length of the epochs, i.e., a new epoch starts if the length of the current epoch exceeds the length of the previous epoch. This leads to a bound of $\order(\sqrt{T_K})$ on the number of epochs which translates to a final regret bound of $\order(K^{2/3})$ in SSP models. However, our first stopping criterion allows us to bound the number of epochs by $\order(\sqrt K)$ rather than $\order(\sqrt{T_K})$ (see Lemma~\ref{lem: number of epochs}). This is one of the key steps in avoiding dependency on $\cmin^{-1}$ (i.e., a lower bound on the cost function) in the main term of the regret and achieve a final regret bound of $\order(\sqrt{K})$.
\begin{remark}
The \ssp~algorithm  only requires the knowledge of the prior distribution $\mu_1$. It does not require the knowledge of $\B$ and $\T$ (an upper bound on the expected time the optimal policy takes to reach the goal) as in \citet{cohen2021minimax}.
\end{remark}


\begin{algorithm}
\caption{\textsc{\ssp}}
\label{alg: posterior sampling}
\textbf{Input: } $\mu_1$\\
\textbf{Initialization: }$t \gets 1, \ell \gets 0, K_{-1} \gets 0, t_0 \gets 0, k_{t_0} \gets 0$\\
\For{ {\normalfont episodes} $k=1, 2, \cdots, K$}{
	$s_t \gets \sinit$ \\
	\While{$s_t \neq g$}{
		\If{$k - k_{t_\ell} > K_{\ell-1}$ or $n_{t}(s, a) > 2 n_{t_\ell}(s, a)$ for some $(s, a) \in \calS \times \calA$}{
			$K_{\ell} \gets k - k_{t_\ell}$\\
			$\ell \gets \ell + 1$ \\
			$t_\ell \gets t$\\			
			$k_{t_\ell} \gets k$\\	
			Generate $\theta_\ell \sim \mu_{t_\ell}(\cdot)$ and compute $\pi_\ell(\cdot) = \pi^*(\cdot;\theta_\ell)$ according to \eqref{eq: Bellman equation} \\	
		}
		Choose action $a_t = \pi_\ell(s_t)$ and observe $s_t' \sim \theta_*(\cdot | s_t, a_t)$ \\
		Update $\mu_{t+1}$ according to \eqref{eq: update rule}\\
		$s_{t+1} \gets s_t'$\\
		$t \gets t+1$ \\
	}
}
\end{algorithm}

\paragraph*{Main Results.} We now provide our main results for the \ssp ~algorithm for unknown SSP models. Our first result considers the case where the cost function is strictly positive for all state-action pairs. Subsequently, we extend the result to the general case by adding a small positive perturbation to the cost function and running the algorithm with the perturbed costs. We first assume that
\begin{assumption}
	\label{ass: cmin}
	There exists $\cmin > 0$, such that $c(s, a) \geq \cmin$ for all state-action pairs $(s, a)$.
\end{assumption}
This assumption allows us to bound the total time spent in $K$ episodes with the total cost, i.e., $\cmin T_K \leq C_K$, where $C_K := \sum_{t=1}^{T_K} c(s_t, a_t)$ is the total cost during the $K$ episodes. To facilitate the presentation of the results, we assume that $S \geq 2$, $A \geq 2$, and $K \geq S^2A$. The first main result is as follows.
\begin{theorem}
	\label{thm1}
	Suppose Assumptions~\ref{ass: class of ssp} and ~\ref{ass: cmin} hold. Then, the regret of the \ssp~algorithm is upper bounded as
	\begin{align*}
		R_K = \order\rbr{\B S \sqrt{KA}L^2 + S^2A \sqrt{\frac{{\B}^3}{\cmin}}L^2},
	\end{align*}
	where $L = \log (\B SAK\cmininv)$.
\end{theorem}
Note that when $K \gg \B S^2A\cmininv$, the regret bound scales as $\otil(\B S \sqrt{KA})$. A crucial point about the above result is that the dependency on $\cmininv$ is only in the lower order term. This allows us to extend the $\order(\sqrt{K})$ bound to the general case where Assumption~\ref{ass: cmin} does not hold by using the perturbation technique of \cite{rosenberg2020near} (see Theorem~\ref{thm2}). Avoiding dependency on $\cmininv$ in the main term is achieved by using a Bernstein-type confidence set in the analysis inspired by \cite{rosenberg2020near}. We note that using a Hoeffding-type confidence set in the analysis as in \citet{ouyang2017learning} gives a regret bound of $\order(\sqrt{K/\cmin})$ which results in $\order(K^{2/3})$ regret bound if Assumption~\ref{ass: cmin} is violated.
\begin{theorem}
	\label{thm2}
	Suppose Assumption~\ref{ass: class of ssp} holds. Running the \ssp~algorithm with costs $c_\epsilon(s, a) := \max \{c(s, a), \epsilon\}$ for $\epsilon = (S^2A/K)^{2/3}$ yields
	\begin{align*}
		R_K = \order\rbr{\B S \sqrt{KA}\tilde{L}^2 + (S^2A)^\frac{2}{3}K^\frac{1}{3}(\B^\frac{3}{2}\tilde{L}^2 + \T) + S^2A\T^\frac{3}{2}\tilde{L}^2},
	\end{align*}
	where $\tilde L := \log (K\B\T SA)$. 
\end{theorem}

Note that when $K \gg S^2A(\B^3 + \T(\T/\B)^6)$, the regret bound scales as $\otil(\B S \sqrt{KA})$. These results have similar regret bounds as the \texttt{Bernstein-SSP} algorithm \citep{rosenberg2020near}, and have a gap of $\sqrt{S}$ with the lower bound of $\Omega(\B\sqrt{SAK})$.

\section{Theoretical Analysis}
\label{sec: analysis}
In this section, we prove Theorem~\ref{thm1}. Proof of Theorem~\ref{thm2} can be found in the Appendix. 

A key property of posterior sampling is that conditioned on the information at time $t$, $\theta_*$ and $\theta_t$ have the same distribution if $\theta_t$ is sampled from the posterior distribution at time $t$ \citep{osband2013more,russo2014learning}. Since the \ssp~algorithm~samples $\theta_\ell$ at the stopping time $t_\ell$, we use the stopping time version of the posterior sampling property stated as follows.
\begin{lemma}[Adapted from Lemma 2 of \cite{ouyang2017learning}]
\label{lem: property of ps}
Let $t_\ell$ be a stopping time with respect to the filtration $(\calF_t)_{t=1}^\infty$, and $\theta_\ell$ be the sample drawn from the posterior distribution at time $t_\ell$. Then, for any measurable function $f$ and any $\calF_{t_\ell}$-measurable random variable $X$, we have
\begin{align*}
\E[f(\theta_\ell, X)|\calF_{t_\ell}] = \E[f(\theta_*, X)|\calF_{t_\ell}].
\end{align*}
\end{lemma}

We now sketch the proof of Theorem~\ref{thm1}. Let $0 < \delta < 1$ be a parameter to be chosen later. We distinguish between \textit{known} and \textit{unknown} state-action pairs. A state-action pair $(s, a)$ is \textit{known} if the number of visits to $(s, a)$ is at least $\alpha \cdot \frac{\B S}{\cmin}\log \frac{\B SA}{\delta \cmin}$ for some large enough constant $\alpha$ (to be determined in Lemma~\ref{lem: known state-action}), and \textit{unknown} otherwise. We divide each epoch into \textit{intervals}. The first interval starts at time $t = 1$. Each interval ends if any of the following conditions hold: (i) the total cost during the interval is at least $\B$; (ii) an unknown state-action pair is met; (iii) the goal state is reached; or (iv) the current epoch completes. The idea of introducing intervals is that after all state-action pairs are known, the cost accumulated during an interval is at least $\B$ (ignoring conditions (iii) and (iv)), which allows us to bound the number of intervals with the total cost divided by $\B$. Note that introducing intervals and distinguishing between known and unknown state-action pairs is only in the analysis and thus knowledge of $\B$ is not required.

Instead of bounding $R_K$, we bound $R_M$ defined as
\begin{align*}
R_M &:= \E\sbr{\sum_{t=1}^{T_M} c(s_t, a_t) - KV(\sinit;\theta_*)},
\end{align*}
for any number of intervals $M$ as long as $K$ episodes are not completed. Here, $T_M$ is the total time of the first $M$ intervals. Let $C_M$ denote the total cost of the algorithm after $M$ intervals and define $L_M$ as the number of epochs in the first $M$ intervals. Observe that the number of times conditions (i), (ii), (iii), and (iv) trigger to start a new interval are bounded by $C_{M}/{\B}$, $\order(\frac{\B S^2A}{\cmin}\log \frac{\B SA}{\delta \cmin})$, $K$, and $L_M$, respectively. Therefore, number of intervals can be bounded as
\begin{align}
\label{eq: bound on m}
M \leq \frac{C_{M}}{\B} + K + L_M + \order(\frac{\B S^2A}{\cmin}\log \frac{\B SA}{\delta \cmin}).
\end{align}
Moreover, since the cost function is lower bounded by $\cmin$, we have $\cmin T_M \leq C_M$. Our argument proceeds as follows.\footnote{Lower order terms are neglected.} We bound $R_M \lesssim \B S\sqrt{MA}$ which implies $\E[C_M] \lesssim K\E[V(\sinit;\theta_*)] + \B S\sqrt{MA}$. From the definition of intervals and once all the state-action pairs are known, the cost accumulated within each interval is at least $\B$ (ignoring intervals that end when the epoch or episode ends). This allows us to bound the number of intervals $M$ with $C_M/\B$ (or $\E[C_M]/\B$). Solving for $\E[C_M]$ in the quadratic inequality $\E[C_M] \lesssim K\E[V(\sinit;\theta_*)] + \B S\sqrt{MA} \lesssim K\E[V(\sinit;\theta_*)] + S\sqrt{\E[C_M]\B A}$ implies that $\E[C_M] \lesssim K\E[V(\sinit;\theta_*)] + \B S\sqrt{AK}$. Since this bound holds for any number of $M$ intervals as long as $K$ episodes are not passed, it holds for $\E[C_K]$ as well. Moreover, since $\cmin > 0$, this implies that the $K$ episodes eventually terminate and proves the final regret bound.

\textbf{Bounding the Number of Epochs.} Before proceeding with bounding $R_M$, we first prove that the number of epochs is bounded as $\order(\sqrt{KSA\log T_M})$. Recall that the length of the epochs is determined by two stopping criteria. If we ignore the second criterion for a moment, the first stopping criterion ensures that the number of episodes within each epoch grows at a linear rate which implies that the number of epochs is bounded by $\order(\sqrt{K})$. If we ignore the first stopping criterion for a moment, the second stopping criterion triggers at most $\order(SA\log T_M)$ times. The following lemma shows that the number of epochs remains of the same order even if  these two criteria are considered simultaneously.
\begin{lemma}
\label{lem: number of epochs}
The number of epochs is bounded as $L_M \leq  \sqrt{2SAK\log T_M} + SA\log T_M$.
\end{lemma}

We now provide the proof sketch for bounding $R_M$. With abuse of notation define $t_{L_M+1} := T_M+1$. We can write
\begin{align}
R_M &:= \E\sbr{\sum_{t=1}^{T_M} c(s_t, a_t) - KV(\sinit;\theta_*)} = \E\sbr{\sum_{\ell = 1}^{L_M}\sum_{t=t_\ell}^{t_{\ell+1}-1} c(s_t, a_t)} - K\E\sbr{V(\sinit;\theta_*)}.
\end{align}
Note that within epoch $\ell$, action $a_t$ is taken according to the optimal policy with respect to $\theta_\ell$. Thus, with the Bellman equation we can write
\begin{align*}
c(s_t, a_t) = V(s_t;\theta_\ell) - \sum_{s'}\theta_\ell(s'|s_t, a_t)V(s';\theta_\ell).
\end{align*}
Substituting this and adding and subtracting $V(s_{t+1};\theta_\ell)$ and $V(s'_t;\theta_\ell)$, decomposes $R_M$ as
\begin{align*}
R_M = R_M^1 + R_M^2 + R_M^3,
\end{align*}
where
\begin{align*}
R_M^1 &:= \E\sbr{\sum_{\ell = 1}^{L_M}\sum_{t=t_\ell}^{t_{\ell+1}-1}\sbr{V(s_t;\theta_\ell) - V(s_{t+1};\theta_\ell)}} \\
R_M^2 &:= \E\sbr{\sum_{\ell = 1}^{L_M}\sum_{t=t_\ell}^{t_{\ell+1}-1}\sbr{V(s_{t+1};\theta_\ell) - V(s'_t;\theta_\ell)}} - K\E\sbr{V(\sinit;\theta_*)}  \\
R_M^3 &:= \E\sbr{\sum_{\ell = 1}^{L_M}\sum_{t=t_\ell}^{t_{\ell+1}-1}\sbr{V(s'_t;\theta_\ell) - \sum_{s'}\theta_\ell(s'|s_t, a_t)V(s';\theta_\ell)}}.
\end{align*}
We proceed by bounding these terms separately. Proof of these lemmas can be found in the supplementary material. $R_M^1$ is a telescopic sum and can be bounded by the following lemma.
\begin{lemma}
\label{lem: bounding R1}
The first term $R_M^1$ is bounded as $R_M^1 \leq \B \E[L_M]$.
\end{lemma}
To bound $R_M^2$, recall that $s'_t \in \calS^+$ is the next state of the environment after applying action $a_t$ at state $s_t$, and that $s'_t = s_{t+1}$ for all time steps except the last time step of an episode (right before reaching the goal). In the last time step of an episode, $s'_t = g$ while $s_{t+1} = \sinit$. This proves that the inner sum of $R_M^2$ can be written as $V(\sinit;\theta_\ell)K_\ell$, where $K_\ell$ is the number of visits to the goal state during epoch $\ell$. Using $K_\ell \leq K_{\ell-1}+1$ and the property of posterior sampling completes the proof. This is formally stated in the following lemma.
\begin{lemma}
\label{lem: r2}
The second term $R_M^2$ is bounded as $R_M^2 \leq \B \E[L_M]$.
\end{lemma}
The rest of the proof proceeds to bound the third term $R_M^3$ which contributes to the dominant term of the final regret bound. The detailed proof can be found in Lemma~\ref{lem: r3}. Here we provide the proof sketch. $R_M^3$ captures the difference between $V(\cdot;\theta_\ell)$ at the next state $s_t' \sim \theta_*(\cdot|s_t, a_t)$ and its expectation with respect to the sampled $\theta_\ell$. Applying the Hoeffding-type concentration bounds \citep{weissman2003inequalities}, as used by \cite{ouyang2017learning} yields a regret bound of $\order(K^{2/3})$ which is sub-optimal. To achieve the optimal dependency on $K$, we use a technique based on the Bernstein concentration bound inspired by the work of \cite{rosenberg2020near}. This requires a more careful analysis. Let $n_{t_\ell}(s, a, s')$ be the number of visits to state-action pair $(s, a)$ followed by state $s'$ before time $t_\ell$. For a fixed state-action pair $(s, a)$, define the Bernstein confidence set using the empirical transition probability $\hatthetal(s'|s, a) := \frac{n_{t_\ell}(s, a, s')}{n_{t_\ell}(s, a)}$ as
\begin{align}
\label{eq: bernstein confidence set}
B_\ell(s, a) := \cbr{\theta(\cdot|s, a) : \abs{\theta(s'|s, a) - \hatthetal(s' | s, a)} \leq 4\sqrt{\hatthetal(s'|s, a)\Alsa} + 28\Alsa, \forall s' \in \calS^+}.
\end{align}
Here $\Alsa := \Aldef$ and $n_\ell^+(s, a) := \max \{n_{t_\ell}(s, a), 1\}$. This confidence set is similar to the one used by \citet{rosenberg2020near} and contains the true transition probability $\theta_*(\cdot|s, a)$ with high probability (see Lemma~\ref{lem: high prob bernstein}). Note that $B_\ell(s, a)$ is  $\calF_{t_\ell}$-measurable which allows us to use the property of posterior sampling (Lemma~\ref{lem: property of ps}) to conclude that $B_\ell(s, a)$ contains the sampled transition probability $\theta_\ell(\cdot|s, a)$ as well with high probability. With some algebraic manipulation, $R_M^3$ can be written as (with abuse of notation $\ell := \ell(t)$ is the epoch at time $t$)
\begin{align*}
R_M^3 = \E\sbr{\sum_{t=1}^{T_M}\sumsp\sbr{\theta_*(s'|s_t, a_t) - \thetalm(s'|s_t, a_t)}\rbr{V(s';\thetalm) - \sumsdp \theta_*(s'' | s_t, a_t)V(s''; \thetalm)}}.
\end{align*}
Under the event that both $\theta_*(\cdot|s_t, a_t)$ and $\theta_\ell(\cdot|s_t, a_t)$ belong to the confidence set $B_\ell(s_t, a_t)$, Bernstein bound can be applied to obtain
\begin{align*}
R_M^3 \approx \order\rbr{\E\sbr{\sum_{t=1}^{T_M}\sqrt{S\Alstat\Vlstat}}} = \order\rbr{\summ\E\sbr{\sumtm\sqrt{S\Alstat\Vlstat}}},
\end{align*}
where $t_m$ denotes the start time of interval $m$ and $\mathbb{V}_\ell$ is the empirical variance defined as
\begin{align}
\label{eq: empirical variance}
\Vlstat := \sumsp \theta_*(s'|s_t, a_t)\rbr{V(s';\thetalm) - \sumsdp\theta_*(s''|s_t, a_t)V(s''; \thetalm)}^2. 
\end{align}
Applying Cauchy Schwarz on the inner sum twice implies that
\begin{align*}
R_M^3 \approx \order\rbr{\summ \rbr{\sqrt{S\E\sbr{\sumtm\Alstat }} \cdot \sqrt{\E\sbr{\sumtm\Vlstat}}}}
\end{align*}
Using the fact that all the state-action pairs $(s_t, a_t)$ within an interval except possibly the first one are known, and that the cumulative cost within an interval is at most $2\B$, one can bound $\E\sbr{\sumtm\Vlstat} \  = \order(\B^2)$ (see Lemma~\ref{lem: sum of variance} for details). Applying Cauchy Schwarz again implies
\begin{align*}
R_M^3 \approx \order\rbr{\B \sqrt{MS \E\sbr{\sum_{t=1}^{T_M}\Alstat}}} \approx \order\rbr{\B S\sqrt{MA}}.
\end{align*}
This argument is formally presented in the following lemma.
\begin{lemma}
\label{lem: r3}
The third term $R_M^3$ can be bounded as
\begin{align*}
R_M^3 \leq 288\B S \sqrt{MA \log^2\frac{SA\E[T_M]}{\delta}} + 1632 \B S^2A\log^2\frac{SA\E[T_M]}{\delta} + 4S\B \delta \E[L_M].
\end{align*}
\end{lemma}

Detailed proofs of all lemmas and the theorem can be found in the appendix in the supplementary material.

\section{Experiments}
\label{sec: experiments}

In this section, the performance of our \ssp~algorithm is compared with existing OFU-type algorithms in the literature. Two environments are considered: RandomMDP and GridWorld. RandomMDP \citep{ouyang2017learning,wei2020model} is an SSP with 8 states and 2 actions whose transition kernel and cost function are generated uniformly at random. GridWorld \citep{tarbouriech2020no} is a $3\times 4$ grid (total of 12 states including the goal state) and 4 actions (LEFT, RIGHT, UP, DOWN) with $c(s, a) = 1$ for any state-action pair $(s, a) \in \calS \times \calA$. The agent starts from the initial state located at the top left corner of the grid, and ends in the goal state at the bottom right corner. At each time step, the agent attempts to move in one of the four directions. However, the attempt is successful only with probability 0.85. With probability 0.15, the agent takes any of the undesired directions uniformly at random. If the agent tries to move out of the boundary, the attempt will not be successful and it remains in the same position.

In the experiments, we evaluate the frequentist regret of \ssp~for a fixed environment (i.e., the environment is not sampled from a prior distribution). 
A Dirichlet prior with parameters $[0.1, \cdots, 0.1]$ is considered for the transition kernel. Dirichlet is a common prior in Bayesian statistics since it is a conjugate prior for categorical and multinomial distributions.

We compare the performance of our proposed \ssp~against existing online learning algorithms for the SSP problem (\texttt{UC-SSP} \citep{tarbouriech2020no}, \texttt{Bernstein-SSP} \citep{rosenberg2020near}, \texttt{ULCVI} \citep{cohen2021minimax}, and \texttt{EB-SSP} \citep{tarbouriech2021stochastic}). The algorithms are evaluated at $K= 10,000$ episodes and the results are averaged over 10 runs. 95\% confidence interval is considered to compare the performance of the algorithms. All the experiments are performed on a 2015 Macbook Pro with 2.7 GHz Dual-Core Intel Core i5 processor and 16GB RAM.

\begin{figure}[t]
	\centering
	\begin{tabular}{cc}
		\includegraphics[width=0.5\textwidth]{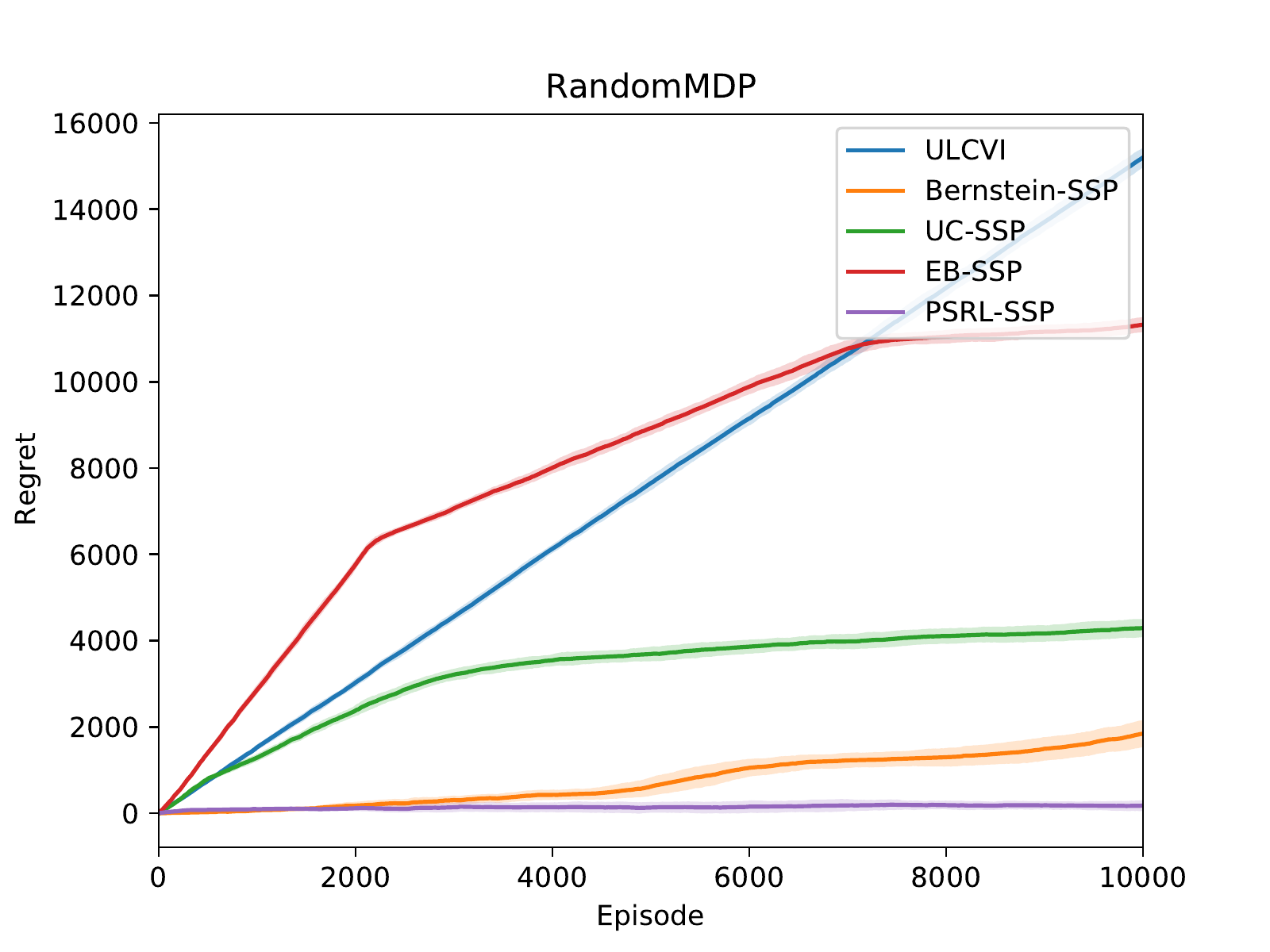} &
		\includegraphics[width=0.5\textwidth]{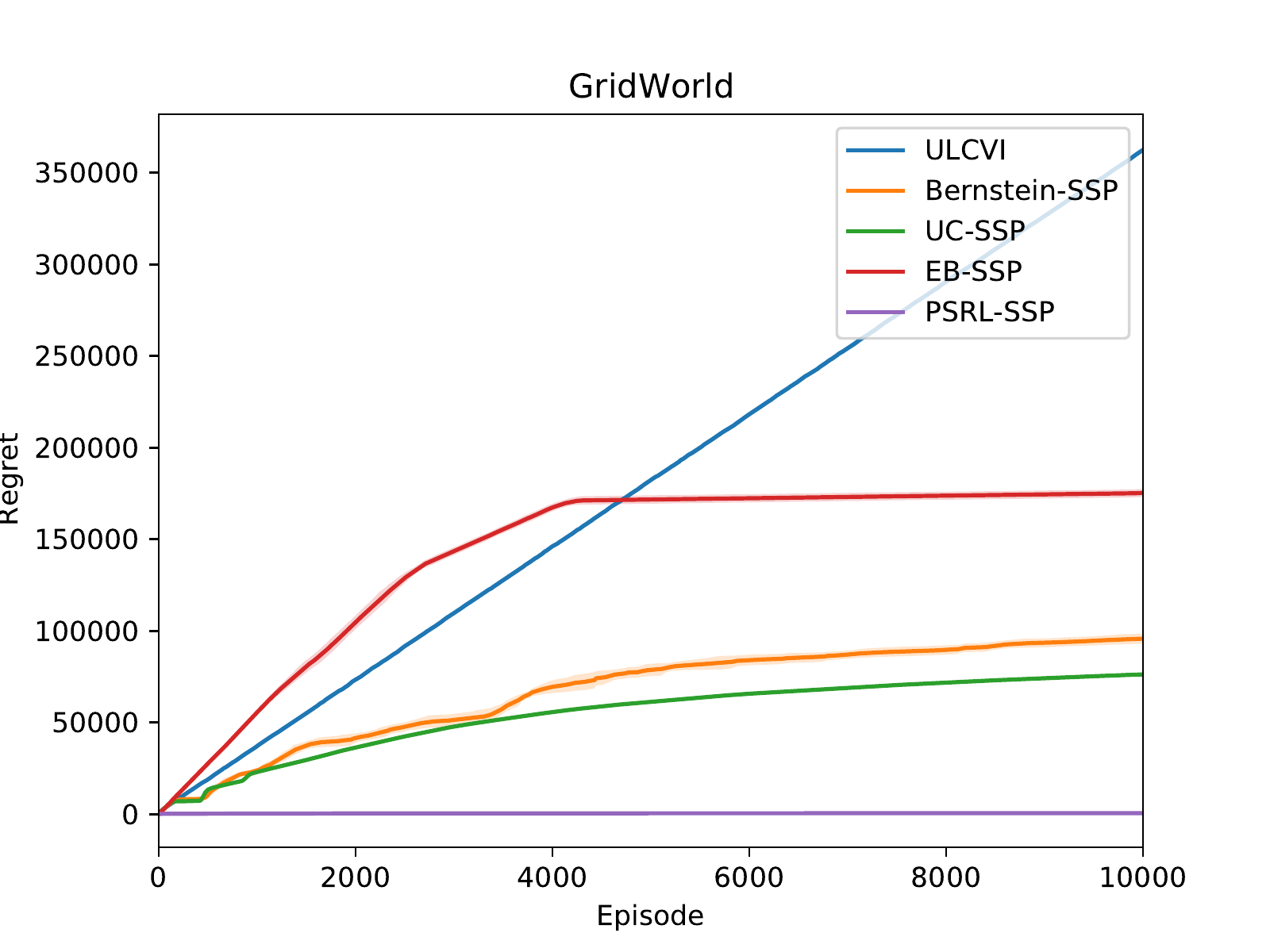}
	\end{tabular}
	\caption{
		Cumulative regret of existing SSP algorithms on RandomMDP (left) and GridWorld (right) for $10,000$ episodes. The results are averaged over 10 runs and 95\% confidence interval is shown with the shaded area. Our proposed \ssp~algorithm outperforms all the existing algorithms considerably. The performance gap is even more significant in the more challenging GridWorld environment (right).
	}
	\label{fig: plot}
\end{figure}

Figure~\ref{fig: plot} shows that \ssp~outperforms all the previously proposed algorithms for the SSP problem, significantly. In particular, it outperforms the recently proposed \texttt{ULCVI} \citep{cohen2021minimax} and \texttt{EB-SSP} \citep{tarbouriech2021stochastic} which match the theoretical lower bound. Our numerical evaluation reveals that the \texttt{ULCVI} algorithm does not show any evidence of learning even after 80,000 episodes (not shown here). 
The poor performance of these algorithms ensures the necessity to consider PS algorithms in practice.

The gap between the performance of \ssp~and OFU algorithms is even more apparent in the GridWorld environment which is more challenging compared to RandomMDP. Note that in RandomMDP, it is possible to go to the goal state from any state with only one step. This is since the transition kernel is generated uniformly at random. However, in the GridWorld environment, the agent has to take a sequence of actions to the right and down to reach the goal at the bottom right corner. Figure~\ref{fig: plot}(right) verifies that \ssp~is able to learn this pattern significantly faster than OFU algorithms.

Since these plots are generated for a fixed environment (not generated from a prior), we conjecture that \ssp~enjoyed the same regret bound under the non-Bayesian setting.

\section*{Conclusions}

In this paper, we have proposed the first posterior sampling-based reinforcement learning algorithm for the SSP models with unknown transition probabilities. The algorithm is very simple as compared to the optimism-based algorithm proposed for SSP models recently \citep{tarbouriech2020no,rosenberg2020near,cohen2021minimax,tarbouriech2021stochastic}. It achieves a  Bayesian regret bound of $\otil(\B S\sqrt{AK})$, where $\B$ is an upper bound on the expected cost of the optimal policy, $S$ is the size of the state space, $A$ is the size of the action space, and $K$ is the number of episodes. This has a $\sqrt{S}$ gap from the best known bound for an optimism-based algorithm but numerical experiments suggest a better performance in practice. A next step would be to extend the algorithm to continuous state and action spaces, and to propose model-free algorithms for such settings. Designing posterior sampling-based model-free algorithms for even average MDPs remains an open problem.

\newpage
\bibliographystyle{plainnat}
\bibliography{online_rl}

\newpage
\appendix


\section{Proofs}
\label{app: proofs}

\subsection{Proof of Lemma~\ref{lem: number of epochs}}\label{app:proof:lem: number of epochs}
\textbf{Lemma} (restatement of Lemma~\ref{lem: number of epochs})\textbf{.} The number of epochs is bounded as $L_M \leq  \sqrt{2SAK\log T_M} + SA\log T_M$.
\begin{proof}
Define macro epoch $i$ with start time $t_{u_i}$ given by $t_{u_1} = t_1$, and
\begin{align*}
t_{u_{i+1}} = \min \cbr{t_\ell > t_{u_i} : n_{t_\ell}(s, a) > 2n_{t_\ell-1}(s, a) \text{ for some } (s, a)}, \qquad i = 2, 3, \cdots.
\end{align*}
A macro epoch starts when the second criterion of determining epoch length triggers. Let $N_M$ be a random variable denoting the total number of macro epochs by the end of interval $M$ and define $u_{N_M + 1} := L_M + 1$. 

Recall that $K_\ell$ is the number of visits to the goal state in epoch $\ell$. Let $\tilde{K}_i := \sum_{\ell=u_i}^{u_{i+1}-1}K_\ell$ be the number of visits to the goal state in macro epoch $i$. By definition of macro epochs, all the epochs within a macro epoch except the last one are triggered by the first criterion, i.e., $K_\ell = K_{\ell-1}+1$ for $\ell = u_i, \cdots, u_{i+1}-2$. Thus,
\begin{align*}
\tilde{K_i} = \sum_{\ell=u_i}^{u_{i+1}-1}K_\ell = K_{u_{i+1}-1} + \sum_{j=1}^{u_{i+1}-u_i - 1}(K_{u_i-1}+j) \geq \sum_{j=1}^{u_{i+1}-u_i - 1}j = \frac{(u_{i+1}-u_i - 1)(u_{i+1}-u_i)}{2}.
\end{align*}
Solving for $u_{i+1}-u_i$ implies that $u_{i+1}-u_i \leq 1 + \sqrt{2\tilde{K_i}}$. We can write
\begin{align*}
L_M = u_{N_M+1} - 1 = \sum_{i=1}^{N_M}\rbr{u_{i+1}-u_i} &\leq \sum_{i=1}^{N_M}\rbr{1 + \sqrt{2\tilde{K_i}}} = N_M + \sum_{i=1}^{N_M}\sqrt{2\tilde{K_i}} \\
&\leq N_M + \sqrt{2N_M \sum_{i=1}^{N_M}\tilde{K_i}} = N_M + \sqrt{2N_M K},
\end{align*}
where the second inequality follows from Cauchy-Schwarz. It suffices to show that the number of macro epochs is bounded as $N_M \leq 1 + SA\log T_M$. Let $\calT_{s, a}$ be the set of all time steps at which the second criterion is triggered for state-action pair $(s, a)$, i.e.,
\begin{align*}
\calT_{s, a} := \cbr{t_\ell \leq T_M : n_{t_\ell}(s, a) > 2 n_{t_{\ell-1}}(s, a)}.
\end{align*}
We claim that $\abr{\calT_{s, a}} \leq \log n_{T_M+1}(s, a)$. To see this, assume by contradiction that $\abr{\calT_{s, a}} \geq 1 + \log n_{T_M+1}(s, a)$, then
\begin{align*}
n_{t_{L_M}}(s, a) &= \prod_{t_\ell \leq T_M, n_{t_{\ell-1}}(s, a) \geq 1}\frac{n_{t_\ell}(s, a)}{n_{t_{\ell-1}}(s, a)} \geq \prod_{t_\ell \in \calT_{s, a}, n_{t_{\ell-1}}(s, a) \geq 1}\frac{n_{t_\ell}(s, a)}{n_{t_{\ell-1}}(s, a)} \\
&> 2^{\abr{\calT_{s, a}} - 1} \geq n_{T_M+1}(s, a),
\end{align*}
which is a contradiction. Thus, $\abr{\calT_{s, a}} \leq \log n_{T_M+1}(s, a)$ for all $(s, a)$. In the above argument, the first inequality is by the fact that $n_t(s, a)$ is non-decreasing in $t$, and the second inequality is by the definition of $\calT_{s, a}$. Now, we can write
\begin{align*}
N_M &= 1 + \sum_{s, a}\abr{\calT_{s, a}} \leq 1 + \sum_{s, a} \log n_{T_M+1}(s, a) \\
&\leq 1 + SA\log \frac{\sum_{s, a}n_{T_M+1}(s, a)}{SA} = 1 + SA \log \frac{T_M}{SA} \leq SA\log T_M,
\end{align*}
where the second inequality follows from Jensen's inequality.
\end{proof}

\subsection{Proof of Lemma~\ref{lem: bounding R1}}\label{app:proof:lem: bounding R1}
\textbf{Lemma} (restatement of Lemma~\ref{lem: bounding R1})\textbf{.} The first term $R_M^1$ is bounded as $R_M^1 \leq \B \E[L_M]$.
\begin{proof}
Recall
\begin{align*}
R_M^1 = \E\sbr{\sum_{\ell = 1}^{L_M}\sum_{t=t_\ell}^{t_{\ell+1}-1}\sbr{V(s_t;\theta_\ell) - V(s_{t+1};\theta_\ell)}} 
\end{align*}
Observe that the inner sum is a telescopic sum, thus
\begin{align*}
R_M^1 = \E\sbr{\sum_{\ell = 1}^{L_M}\sbr{V(s_{t_\ell};\theta_\ell) - V(s_{t_{\ell+1}};\theta_\ell)}} \leq \B \E[L_M],
\end{align*}
where the inequality is by Assumption~\ref{ass: class of ssp}.
\end{proof}

\subsection{Proof of Lemma~\ref{lem: r2}}\label{sec:proof:lem: r2}
\textbf{Lemma} (restatement of Lemma~\ref{lem: r2})\textbf{.} The second term $R_M^2$ is bounded as $R_M^2 \leq \B \E[L_M]$.
\begin{proof}
Recall that $K_\ell$ is the number of times the goal state is reached during epoch $\ell$. By definition, the only time steps that $s'_t \neq s_{t+1}$ is right before reaching the goal. Thus, with $V(g;\theta_\ell) = 0$, we can write
\begin{align*}
R_M^2 &= \E\sbr{\sum_{\ell = 1}^{L_M}\sum_{t=t_\ell}^{t_{\ell+1}-1}\sbr{V(s_{t+1};\theta_\ell) - V(s'_t;\theta_\ell)}} - K\E\sbr{V(\sinit;\theta_*)} \\
&= \E\sbr{\sum_{\ell = 1}^{L_M}V(\sinit;\theta_\ell)K_\ell} - K\E\sbr{V(\sinit;\theta_*)} \\
&= \sum_{\ell = 1}^{\infty}\E\sbr{\one_{\{m(t_\ell) \leq M\}}V(\sinit;\theta_\ell)K_\ell} - K\E\sbr{V(\sinit;\theta_*)},
\end{align*}
where the last step is by Monotone Convergence Theorem. Here $m(t_\ell)$ is the interval at time $t_\ell$. Note that from the first stopping criterion of the algorithm we have $K_\ell \leq K_{\ell-1}+1$ for all $\ell$. Thus, each term in the summation can be bounded as
\begin{align*}
\E\sbr{\one_{\{m(t_\ell) \leq M\}}V(\sinit;\theta_\ell)K_\ell} \leq \E\sbr{\one_{\{m(t_\ell) \leq M\}}V(\sinit;\theta_\ell)(K_{\ell-1} + 1)}.
\end{align*}
$\one_{\{m(t_\ell) \leq M\}}(K_{\ell-1} + 1)$ is $\calF_{t_\ell}$ measurable. Therefore, applying the property of posterior sampling (Lemma~\ref{lem: property of ps}) implies
\begin{align*}
\E\sbr{\one_{\{m(t_\ell) \leq M\}}V(\sinit;\theta_\ell)(K_{\ell-1} + 1)} = \E\sbr{\one_{\{m(t_\ell) \leq M\}}V(\sinit;\theta_*)(K_{\ell-1} + 1)}
\end{align*}
Substituting this into $R_M^2$, we obtain
\begin{align*}
R_M^2 &\leq \sum_{\ell = 1}^{\infty}\E\sbr{\one_{\{m(t_\ell) \leq M\}}V(\sinit;\theta_*)(K_{\ell-1}+1)} - K\E\sbr{V(\sinit;\theta_*)} \\
&= \E\sbr{\sum_{\ell = 1}^{L_M}V(\sinit;\theta_*)(K_{\ell-1}+1)} - K\E\sbr{V(\sinit;\theta_*)} \\
&= \E\sbr{V(\sinit;\theta_*)\rbr{\sum_{\ell = 1}^{L_M}K_{\ell-1} - K}} + \E\sbr{V(\sinit;\theta_*)L_M} \leq \B\E[L_M].
\end{align*}
In the last inequality we have used the fact that $0 \leq V(\sinit;\theta_*) \leq \B$ and $\sum_{\ell = 1}^{L_M}K_{\ell-1} \leq K$.
\end{proof}

\subsection{Proof of Lemma~\ref{lem: r3}}\label{app:proof:lem: r3}
\textbf{Lemma} (restatement of Lemma~\ref{lem: r3})\textbf{.} The third term $R_M^3$ can be bounded as
\begin{align*}
R_M^3 \leq 288\B S \sqrt{MA \log^2\frac{SA\E[T_M]}{\delta}} + 1632 \B S^2A\log^2\frac{SA\E[T_M]}{\delta} + 4S\B \delta \E[L_M].
\end{align*}
\begin{proof}
With abuse of notation let $\ell := \ell(t)$ denote the epoch at time $t$ and $m(t)$ be the interval at time $t$. We can write
\begin{align*}
&R_M^3 = \E\sbr{\sum_{t=1}^{T_M}\sbr{V(s'_t;\theta_\ell) - \sum_{s'}\thetalm(s'|s_t, a_t)V(s';\thetalm)}} \\
&= \E\sbr{\sum_{t=1}^{\infty}\one_{\{m(t) \leq M\}}\sbr{V(s'_t;\theta_\ell) - \sum_{s'}\thetalm(s'|s_t, a_t)V(s';\thetalm)}} \\
&= \sum_{t=1}^{\infty}\E\sbr{\one_{\{m(t) \leq M\}}\E\sbr{V(s'_t;\theta_\ell) - \sum_{s'}\thetalm(s'|s_t, a_t)V(s';\thetalm)\Big| \calF_t, \theta_*, \theta_\ell}}.
\end{align*}
The last equality follows from Dominated Convergence Theorem, tower property of conditional expectation, and that $\one_{\{m(t) \leq M\}}$ is measurable with respect to $\calF_t$. Note that conditioned on $\calF_t$, $\theta_*$ and $\theta_\ell$, the only random variable in the inner expectation is $s'_t$. Thus, $\E[V(s'_t;\theta_\ell) | \calF_t, \theta_*, \theta_\ell] = \sum_{s'}\theta_*(s'|s_t, a_t)V(s';\theta_\ell)$. Using Dominated Convergence Theorem again implies that
\begin{align}
&R_M^3 = \E\sbr{\sum_{t=1}^{T_M}\sumsp\sbr{\theta_*(s'|s_t, a_t) - \thetalm(s'|s_t, a_t)}V(s';\thetalm)} \nonumber \\
&= \E\sbr{\sum_{t=1}^{T_M}\sumsp\sbr{\theta_*(s'|s_t, a_t) - \thetalm(s'|s_t, a_t)}\rbr{V(s';\thetalm) - \sumsdp \theta_*(s'' | s_t, a_t)V(s''; \thetalm)}}, \label{eq: pf lem rr3 tmp1}
\end{align}
where the last equality is due to the fact that $\theta_*(\cdot|s_t, a_t)$ and $\thetalm(\cdot|s_t, a_t)$ are probability distributions and that $\sumsdp \theta_*(s'' | s_t, a_t)V(s''; \thetalm)$ is independent of $s'$.

Recall the Bernstein confidence set $B_\ell(s, a)$ defined in \eqref{eq: bernstein confidence set} and let $\Omega^\ell_{s, a}$ be the event that both $\theta_*(\cdot|s, a)$ and $\theta_\ell(\cdot|s, a)$ are in $B_\ell(s, a)$. If $\Omega^\ell_{s, a}$ holds, then the difference between $\theta_*(\cdot|s, a)$ and $\theta_\ell(\cdot|s, a)$ can be bounded by the following lemma.
\begin{lemma}
Denote $\Alsa = \Aldef$. If $\Omega^\ell_{s, a}$ holds, then 
\begin{align*}
\abr{\theta_*(s'|s, a) - \theta_\ell(s'|s, a)} \leq 8\sqrt{\theta_*(s'|s, a)\Alsa} + 136\Alsa.
\end{align*}
\begin{proof}
Since $\Omega^\ell_{s, a}$ holds, by \eqref{eq: bernstein confidence set} we have that
\begin{align*}
\hatthetal(s' | s, a) - \theta_*(s'|s, a) \leq 4\sqrt{\hatthetal(s'|s, a)\Alsa} + 28\Alsa.
\end{align*}
Using the primary inequality that $x^2 \leq ax + b$ implies $x \leq a + \sqrt b$ with $x = \sqrt{\hatthetal(s' | s, a)}$, $a = 4\sqrt{\Alsa}$, and $b = \theta_*(s'|s, a) + 28\Alsa$, we obtain
\begin{align*}
\sqrt{\hatthetal(s' | s, a)} \leq 4\sqrt{\Alsa} + \sqrt{\theta_*(s'|s, a) + 28\Alsa} \leq \sqrt{\theta_*(s'|s, a)} + 10\sqrt{\Alsa},
\end{align*}
where the last inequality is by sub-linearity of the square root. Substituting this bound into ~\eqref{eq: bernstein confidence set} yields
\begin{align*}
\abs{\theta_*(s'|s, a) - \hatthetal(s' | s, a)} \leq 4\sqrt{\theta_*(s'|s, a)\Alsa} + 68\Alsa.
\end{align*}
Similarly,
\begin{align*}
\abs{\theta_\ell(s'|s, a) - \hatthetal(s' | s, a)} \leq 4\sqrt{\theta_*(s'|s, a)\Alsa} + 68\Alsa.
\end{align*}
Using the triangle inequality completes the proof.
\end{proof}
\label{lem: theta_star minus theta_l}
\end{lemma}
Note that if either of $\theta_*(\cdot|s_t, a_t)$ or $\thetalm(\cdot|s_t, a_t)$ is not in $B_\ell(s_t, a_t)$, then the inner term of \eqref{eq: pf lem rr3 tmp1} can be bounded by $2S\B$ (note that $|\calS^+| \leq 2S$ and $V(\cdot;\thetalm) \leq \B$). Thus, applying Lemma~\ref{lem: theta_star minus theta_l} implies that
\begin{align*}
&\sumsp\sbr{\theta_*(s'|s_t, a_t) - \thetalm(s'|s_t, a_t)}\rbr{V(s';\thetalm) - \sumsdp \theta_*(s'' | s_t, a_t)V(s''; \thetalm)} \\
&\leq 8 \sumsp \sqrt{\Alstat\theta_*(s'|s_t, a_t)\rbr{V(s';\thetalm) - \sumsdp\theta_*(s''|s_t, a_t)V(s''; \thetalm)}^2}\one_{\Omega^\ell_{s_t, a_t}}\\
&\qquad + 136 \sumsp \Alstat \abr{V(s';\thetalm) - \sumsdp\theta_*(s''|s_t, a_t)V(s''; \thetalm)}\one_{\Omega^\ell_{s_t, a_t}} \\
&\qquad + 2S\B\rbr{\one_{\{\theta_*(\cdot|s_t, a_t) \notin B_\ell(s_t, a_t)\}} + \one_{\{\thetalm(\cdot|s_t, a_t) \notin B_\ell(s_t, a_t)\}}} \\
&\leq 16\sqrt{S\Alstat\Vlstat}\one_{\Omega^\ell_{s_t, a_t}} + 272S\B \Alstat\one_{\Omega^\ell_{s_t, a_t}} \\
&\qquad+ 2S\B\big(\one_{\{\theta_*(\cdot|s_t, a_t) \notin B_\ell(s_t, a_t)\}} + \one_{\{\thetalm(\cdot|s_t, a_t) \notin B_\ell(s_t, a_t)\}}\big).
\end{align*}
where $\Alsa = \Aldef$ and $\mathbb{V}_\ell(s, a)$ is defined in \eqref{eq: empirical variance}. Here the last inequality follows from Cauchy-Schwarz, $|\calS^+| \leq 2S$, $V(\cdot;\thetalm) \leq \B$ and the definition of $\mathbb{V}_\ell$. Substituting this into \eqref{eq: pf lem rr3 tmp1} yields
\begin{align}
R_M^3 &\leq 16\sqrt S\E\sbr{\sum_{t=1}^{T_M}\sqrt{\Alstat\Vlstat}\oneomega} \label{eq: pf lem R3 tmp1}\\
&\qquad + 272S\B \E\sbr{\sum_{t=1}^{T_M}\Alstat\oneomega}  \label{eq: pf lem R3 tmp2}\\
&\qquad + 2S\B\E\sbr{\sum_{t=1}^{T_M} \rbr{\one_{\{\theta_*(\cdot|s_t, a_t) \notin B_\ell(s_t, a_t)\}} + \one_{\{\thetalm(\cdot|s_t, a_t) \notin B_\ell(s_t, a_t)\}}}}. \label{eq: pf lem R3 tmp3}
\end{align}
The inner sum in \eqref{eq: pf lem R3 tmp2} is bounded by $6SA \log^2 (SAT_M/\delta)$ (see Lemma~\ref{lem: sumt of Alstat}). To bound \eqref{eq: pf lem R3 tmp3}, we first show that $B_\ell(s, a)$ contains the true transition probability $\theta_*(\cdot|s, a)$ with high probability: 
\begin{lemma}
For any epoch $\ell$ and any state-action pair $(s, a) \in \calS \times \calA$, $\theta_*(\cdot|s, a) \in B_\ell(s, a)$ with probability at least $1 - \frac{\delta}{2SAn^+_\ell(s, a)}$.
\label{lem: high prob bernstein}
\end{lemma}
\begin{proof}
Fix $(s, a, s') \in \calS\times\calA\times\calS^+$ and $0 < \delta' < 1$ (to be chosen later). Let $(Z_i)_{i=1}^\infty$ be a sequence of random variables drawn from the probability distribution $\theta_*(\cdot|s, a)$. Apply Lemma~\ref{lem: D.3 of cohen} below with $X_i = \one_{\{Z_i = s'\}}$ and $\delta_t = \frac{\delta'}{4St^2}$ to a prefix of length $t$ of the sequence $(X_i)_{i=1}^\infty$, and apply union bound over all $t$ and $s'$ to obtain
\begin{align*}
\abr{\hat\theta_\ell(s'|s, a) - \theta_*(s'|s, a)} \leq 2\sqrt{\frac{\hat\theta_\ell(s'|s, a)\log\frac{8S{n^+_\ell}^2(s, a)}{\delta'}}{n^+_\ell(s, a)}} + 7\log \frac{8S{n^+_\ell}^2(s, a)}{\delta'}
\end{align*}
with probability at least $1 - \delta'/2$ for all $s' \in \calS^+$ and $\ell \geq 1$, simultaneously. Choose $\delta'=\delta/SAn^+_{\ell}(s, a)$ and use $S \geq 2$, $A \geq 2$ to complete the proof.
\end{proof}
\begin{lemma}[Theorem D.3 (Anytime Bernstein) of \cite{rosenberg2020near}]
Let $(X_n)_{n=1}^\infty$ be a sequence of independent and identically distributed random variables with expectation $\mu$. Suppose that $0 \leq X_n \leq B$ almost surely. Then with probability at least $1 - \delta$, the following holds for all $n \geq 1$ simultaneously:
\begin{align*}
\abr{\sum_{i=1}^n(X_i - \mu)} \leq 2\sqrt{B\sum_{i=1}^nX_i \log\frac{2n}{\delta}} + 7B\log\frac{2n}{\delta}.
\end{align*}
\label{lem: D.3 of cohen}
\end{lemma}

Now, by rewriting the sum in \eqref{eq: pf lem R3 tmp3} over epochs, we have
\begin{align*}
&\E\sbr{\sum_{t=1}^{T_M} \rbr{\one_{\{\theta_*(\cdot|s_t, a_t) \notin B_\ell(s_t, a_t)\}} + \one_{\{\thetalm(\cdot|s_t, a_t) \notin B_\ell(s_t, a_t)\}}}} \\
&= \E\sbr{\suml \sumtl \rbr{\one_{\{\theta_*(\cdot|s_t, a_t) \notin B_\ell(s_t, a_t)\}} + \one_{\{\thetalm(\cdot|s_t, a_t) \notin B_\ell(s_t, a_t)\}}}} \\
&= \sum_{s, a}\E\sbr{\suml \sumtl \one_{\{s_t=s, a_t=a\}}\rbr{\one_{\{\theta_*(\cdot|s, a) \notin B_\ell(s, a)\}} + \one_{\{\thetalm(\cdot|s, a) \notin B_\ell(s, a)\}}}} \\
&= \sum_{s, a}\E\sbr{\suml \rbr{n_{t_{\ell+1}}(s, a) - n_{t_\ell}(s, a)}\rbr{\one_{\{\theta_*(\cdot|s, a) \notin B_\ell(s, a)\}} + \one_{\{\thetalm(\cdot|s, a) \notin B_\ell(s, a)\}}}}.
\end{align*}
Note that $n_{t_{\ell+1}}(s, a) - n_{t_\ell}(s, a) \leq n_{t_\ell}(s, a)+1$ by the second stopping criterion. Moreover, observe that $B_\ell(s, a)$ is $\calF_{t_\ell}$ measurable. Thus, it follows from the property of posterior sampling (Lemma~\ref{lem: property of ps}) that $\E[\one_{\{\thetalm(\cdot|s, a) \notin B_\ell(s, a)\}}|\calF_{t_\ell}] = \E[\one_{\{\theta_*(\cdot|s, a) \notin B_\ell(s, a)\}}|\calF_{t_\ell}] = \mathbb{P}(\theta_*(\cdot|s, a) \notin B_\ell(s, a)|\calF_{t_\ell}) \leq \delta/(2SAn_\ell^+(s, a))$, where the inequality is by Lemma~\ref{lem: high prob bernstein}. Using Monotone Convergence Theorem and that $\one_{\{m(t_\ell) \leq M\}}$ is $\calF_{t_\ell}$ measurable, we can write
\begin{align*}
&\sum_{s, a}\E\sbr{\suml \rbr{n_{t_{\ell+1}}(s, a) - n_{t_\ell}(s, a)}\rbr{\one_{\{\theta_*(\cdot|s, a) \notin B_\ell(s, a)\}} + \one_{\{\thetalm(\cdot|s, a) \notin B_\ell(s, a)\}}}} \\
&\leq \sum_{s, a}\sum_{\ell=1}^\infty\E\sbr{ \one_{\{m(t_\ell) \leq M\}}  \rbr{n_{t_\ell}(s, a)+1}\E\sbr{\one_{\{\theta_*(\cdot|s, a) \notin B_\ell(s, a)\}} + \one_{\{\thetalm(\cdot|s, a) \notin B_\ell(s, a)\}}| \calF_{t_\ell}}} \\
&\leq \sum_{s, a}\sum_{\ell=1}^\infty\E\sbr{ \one_{\{m(t_\ell) \leq M\}}  \rbr{n_{t_\ell}(s, a)+1}\frac{\delta}{SAn_\ell^+(s, a)}} \\
&\leq 2\delta \E[L_M],
\end{align*}
where the last inequality is by $n_{t_\ell}(s, a)+1 \leq 2n_\ell^+(s, a)$ and Monotone Convergence Theorem.

We proceed by bounding \eqref{eq: pf lem R3 tmp1}. Denote by $t_m$ the start time of interval $m$, define $t_{M+1} := T_M+1$, and rewrite the sum in \eqref{eq: pf lem R3 tmp1} over intervals to get
\begin{align*}
\E\sbr{\sum_{t=1}^{T_M}\sqrt{\Alstat\Vlstat}\oneomega} &= \summ\E\sbr{\sumtm\sqrt{\Alstat\Vlstat}\oneomega}
\end{align*}
Applying Cauchy-Schwarz twice on the inner expectation implies
\begin{align*}
&\E\sbr{\sumtm\sqrt{\Alstat\Vlstat}\oneomega } \\
&\leq \E\sbr{\sqrt{\sumtm\Alstat}\cdot \sqrt{\sumtm\Vlstat\oneomega} } \\
&\leq \sqrt{\E\sbr{\sumtm\Alstat }} \cdot \sqrt{\E\sbr{\sumtm\Vlstat\oneomega }} \\
&\leq 7\B \sqrt{\E\sbr{\sumtm\Alstat }},
\end{align*}
where the last inequality is by Lemma~\ref{lem: sum of variance}. Summing over $M$ intervals and applying Cauchy-Schwarz, we get
\begin{align*}
\sum_{m=1}^{M}\E&\sbr{\sumtm\sqrt{\Alstat\Vlstat}\oneomega } \leq 7\B \sum_{m=1}^{M} \sqrt{\E\sbr{\sumtm\Alstat }} \\
&\leq 7\B \sqrt{M  \summ\E\sbr{\sumtm\Alstat }} \\
&= 7\B \sqrt{M  \E\sbr{\sum_{t=1}^{T_M}\Alstat }} \\
&\leq 18 \B \sqrt{MSA  \E\sbr{\log^2\frac{SAT_M}{\delta}}},
\end{align*}
where the last inequality follows from Lemma~\ref{lem: sumt of Alstat}.
Substituting these bounds in \eqref{eq: pf lem R3 tmp1}, \eqref{eq: pf lem R3 tmp2}, \eqref{eq: pf lem R3 tmp3}, concavity of $\log^2 x$ for $x \geq 3$, and applying Jensen's inequality completes the proof.

\begin{lemma}
$\sum_{t=1}^{T_M} \Alstat \leq 6SA \log^2 (SAT_M/\delta).$ \\
\label{lem: sumt of Alstat}
\end{lemma}
\begin{proof}
Recall $\Alsa = \Aldef$. Denote by $L := \log(SAT_M/\delta)$, an upper bound on the numerator of $\Alstat$. we have
\begin{align*}
\sum_{t=1}^{T_M}\Alstat &\leq \sum_{t=1}^{T_M}\frac{L}{n_\ell^+(s_t, a_t)} = L\sum_{s, a}\sum_{t=1}^{T_M}\frac{\one_{\{s_t=s, a_t=a\}}}{n_\ell^+(s, a)} \\
&\leq 2L\sum_{s, a}\sum_{t=1}^{T_M}\frac{\one_{\{s_t=s, a_t=a\}}}{n_t^+(s, a)} = 2L\sum_{s, a}\one_{\{n_{T_M+1}(s, a) > 0\}} + 2L\sum_{s, a} \sum_{j=1}^{n_{T_M+1}(s, a)-1}\frac{1}{j} \\
&\leq 2LSA + 2L\sum_{s, a}(1 + \log n_{T_M+1}(s, a)) \\
&\leq 4LSA + 2LSA \log T_M \leq 6LSA \log T_M.
\end{align*}
Here the second inequality is by $n_\ell^+(s, a) \geq 0.5 n_t^+(s, a)$ (the second criterion in determining the epoch length), the third inequality is by $\sum_{x=1}^n1/x \leq 1 + \log n$, and the fourth inequality is by $n_{T_M+1}(s, a) \leq T_M$. The proof is complete by noting that $\log T_M \leq L$.
%
\end{proof}


\begin{lemma}
For any interval $m$, $\E[\sumtm \Vlstat \one_{\Omega^\ell}] \leq 44\B^2$.
\label{lem: sum of variance}
\end{lemma}
\begin{proof}
To proceed with the proof, we need the following two technical lemmas.
\begin{lemma}
Let $(s, a)$ be a known state-action pair and $m$ be an interval. If $\Omega^\ell_{s, a}$ holds, then for any state $s' \in \calS^+$,
\begin{align*}
\abs{\theta_*(s'|s, a) - \theta_\ell(s'|s, a)} \leq \frac{1}{8} \sqrt{\frac{\cmin \theta_*(s'|s, a)}{S\B}} + \frac{\cmin}{4S\B}.
\end{align*}
\label{lem: known state-action}
\end{lemma}
\begin{proof}
From Lemma~\ref{lem: theta_star minus theta_l}, we know that if $\Omega^\ell_{s, a}$ holds, then
\begin{align*}
\abr{\theta_*(s'|s, a) - \theta_\ell(s'|s, a)} \leq 8\sqrt{\theta_*(s'|s, a)\Alsa} + 136\Alsa,
\end{align*}
with $\Alsa = \Aldef$. The proof is complete by noting that $\log (x)/x$ is decreasing, and that $n_\ell^+(s, a) \geq \alpha \cdot \frac{\B S}{\cmin} \log \frac{\B SA}{\delta\cmin}$ for some large enough constant $\alpha$ since $(s, a)$ is known.
\end{proof}
\begin{lemma}[Lemma B.15. of \citet{rosenberg2020near}]
Let $(X_t)_{t=1}^\infty$ be a martingale difference sequence adapted to the filtration $(\calF_t)_{t=0}^\infty$. Let $Y_n = (\sum_{t=1}^nX_t)^2 - \sum_{t=1}^n\E[X_t^2|\calF_{t-1}]$. Then $(Y_n)_{n=0}^\infty$ is a martingale, and in particular if $\tau$ is a stopping time such that $\tau \leq c$ almost surely, then $\E[Y_\tau] = 0$.
\label{lem: b15 of cohen}
\end{lemma}
By the definition of the intervals, all the state-action pairs within an interval except possibly the first one are known. Therefore, we bound
\begin{align*}
\E\sbr{\sumtm \Vlstat \oneomega \Big| \calF_{t_m}} = \E\sbr{\mathbb{V}_\ell(s_{t_m}, a_{t_m})\oneomega | \calF_{t_m}} + \E\sbr{\sumtmplus \Vlstat \oneomega \Big| \calF_{t_m}}.
\end{align*}
The first summand is upper bounded by $\B^2$. To bound the second term, define $Z_\ell^t := [V(s'_{t};\thetalm) - \sum_{s' \in \calS}\theta_*(s'|s_t, a_t)V(s';\thetalm)]\oneomega$. Conditioned on $\calF_{t_m}, \theta_*$ and $\thetalm$, $(Z_\ell^t)_{t \geq t_m}$ constitutes a martingale difference sequence with respect to the filtration $(\calF^m_{t+1})_{t \geq t_m}$, where $\calF^m_t$ is the sigma algebra generated by $\{(s_{t_m}, a_{t_m}), \cdots, (s_t, a_t)\}$. Moreover, $t_{m+1}-1$ is a stopping time with respect to $(\calF^m_{t+1})_{t \geq t_m}$ and is bounded by $t_m + 2\B/\cmin$. Therefore, Lemma~\ref{lem: b15 of cohen} implies that
\begin{align}
\label{eq: pf lem bounded v tmp0}
\E\sbr{\sumtmplus \Vlstat \oneomega \Big| \calF_{t_m}, \theta_*, \thetalm} = \E\sbr{\rbr{\sumtmplus Z_\ell^t \oneomega}^2 \Big| \calF_{t_m}, \theta_*, \thetalm}.
\end{align}
We proceed by bounding $\abs{\sumtmplus Z_\ell^t \oneomega}$ in terms of $\sumtmplus \Vlstat \oneomega$ and combine with the left hand side to complete the proof. We have
\begin{align}
&\abr{\sumtmplus Z_\ell^t \oneomega} = \abr{\sumtmplus \sbr{V(s'_{t};\thetalm) - \sum_{s' \in \calS}\theta_*(s'|s_t, a_t)V(s';\thetalm)}\oneomega} \nonumber \\
&\leq \abr{\sumtmplus \sbr{V(s'_{t};\thetalm) - V(s_{t};\thetalm)}} \label{eq: pf lem bounded v tmp1}\\
&+ \abr{\sumtmplus \sbr{V(s_{t};\thetalm) - \sum_{s' \in \calS}\thetalm(s'|s_t, a_t)V(s';\thetalm)}} \label{eq: pf lem bounded v tmp2}\\
& + \abr{\sumtmplus \sumsp \sbr{\thetalm(s'|s_t, a_t) - \theta_*(s'|s_t, a_t)}\rbr{V(s';\thetalm) - \sumsdp \theta_*(s''|s_t, a_t)V(s'';\thetalm)}\oneomega}. \label{eq: pf lem bounded v tmp3}
\end{align}
where \eqref{eq: pf lem bounded v tmp3} is by the fact that $\thetalm(\cdot|s_t, a_t), \theta_*(\cdot|s_t, a_t)$ are probability distributions and $\sumsdp \theta_*(s''|s_t, a_t)V(s'';\thetalm)$ is independent of $s'$ and $V(g;\thetalm) = 0$. \eqref{eq: pf lem bounded v tmp1} is a telescopic sum (recall that $s_{t+1} = s'_t$ if $s'_t \neq g$) and is bounded by $\B$. It follows from the Bellman equation that \eqref{eq: pf lem bounded v tmp2} is equal to $\sumtmplus c(s_t, a_t)$. By definition, the interval ends as soon as the cost accumulates to $\B$ during the interval. Moreover, since $V(\cdot; \thetalm) \leq \B$, the algorithm does not choose an action with instantaneous cost more than $\B$. This implies that $\sumtmplus c(s_t, a_t) \leq 2\B$. To bound \eqref{eq: pf lem bounded v tmp3} we use the Bernstein confidence set, but taking into account that all the state-action pairs in the summation are known, we can use Lemma~\ref{lem: known state-action} to obtain
\begin{align*}
& \sumsp \rbr{\thetalm(s'|s_t, a_t) - \theta_*(s'|s_t, a_t)}\rbr{V(s';\thetalm) - \sumsdp \theta_*(s''|s_t, a_t)V(s'';\thetalm)}\oneomega \\
&\leq  \sumsp \frac{1}{8} \sqrt{\frac{\cmin \theta_*(s'|s_t, a_t)\rbr{V(s';\thetalm) - \sumsdp \theta_*(s''|s_t, a_t)V(s'';\thetalm)}^2\oneomega}{S\B}} \\
&\qquad + \sumsp \frac{\cmin}{4S\B}\abr{V(s';\thetalm) - \sumsdp \theta_*(s''|s_t, a_t)V(s'';\thetalm)} \\
&\leq \frac{1}{4}\sqrt{\frac{\cmin \Vlstat \oneomega}{\B}} + \frac{c(s_t, a_t)}{2}.
\end{align*}
The last inequality follows from Cauchy-Schwarz inequality, $|\calS^+| \leq 2S$, $|V(\cdot;\thetalm)| \leq \B$, and $\cmin \leq c(s_t, a_t)$. Summing over the time steps in interval $m$ and applying Cauchy-Schwarz, we get
\begin{align*}
\sumtmplus \sbr{\frac{1}{4}\sqrt{\frac{\cmin \Vlstat \oneomega}{\B}} + \frac{c(s_t, a_t)}{2}} &\leq \frac{1}{4}\sqrt{(t_{m+1} - t_m)\frac{\cmin \sumtmplus\Vlstat \oneomega}{\B}} \\
&\qquad+ \frac{\sumtmplus c(s_t, a_t)}{2} \\
&\leq \frac{1}{4}\sqrt{2\sumtmplus \Vlstat \oneomega} + \B.
\end{align*}
The last inequality follows from the fact that duration of interval $m$ is at most $2\B/\cmin$ and its cumulative cost is at most $2\B$. Substituting these bounds into \eqref{eq: pf lem bounded v tmp0} implies that
\begin{align*}
\E\sbr{\sumtmplus \Vlstat \oneomega \Big| \calF_{t_m}, \theta_*, \thetalm} &\leq \E\sbr{\rbr{4\B + \frac{1}{4}\sqrt{2\sumtmplus \Vlstat \oneomega}}^2 \Big| \calF_{t_m}, \theta_*, \thetalm} \\
&\leq 32\B^2 + \frac{1}{4} \E\sbr{\sumtmplus \Vlstat \oneomega \Big| \calF_{t_m}, \theta_*, \thetalm},
\end{align*}
where the last inequality is by $(a+b)^2 \leq 2(a^2 + b^2)$ with $b =\frac{1}{4}\sqrt{2\sumtmplus \Vlstat \oneomega}$ and $a = 4\B$. Rearranging implies that $\E\sbr{\sumtmplus \Vlstat \oneomega | \calF_{t_m}, \theta_*, \thetalm} \leq 43\B^2$ and the proof is complete.
\end{proof}
\end{proof}

\subsection{Proof of Theorem~\ref{thm1}}\label{app:proof:thm1}
\textbf{Theorem} (restatement of Theorem~\ref{thm1})\textbf{.}
Suppose Assumptions~\ref{ass: class of ssp} and ~\ref{ass: cmin} hold. Then, the regret bound of the \ssp~algorithm is bounded as
	\begin{align*}
		R_K = \order\rbr{\B S \sqrt{KA}L^2 + S^2A \sqrt{\frac{{\B}^3}{\cmin}}L^2},
	\end{align*}
	where $L = \log (\B SAK\cmininv)$.
\begin{proof}
Denote by $C_M$ the total cost after $M$ intervals. Recall that 
\begin{align*}
&\E[C_M] = K\E[V(\sinit;\theta_*)] + R_M = K\E[V(\sinit;\theta_*)] + R_M^1 + R_M^2 + R_M^3
\end{align*}
Using Lemmas~\ref{lem: bounding R1}, \ref{lem: r2}, and \ref{lem: r3} with $\delta=1/K$ obtains
\begin{align}
\E[C_M]&\leq K\E[V(\sinit;\theta_*)] \nonumber \\
&+ \order\rbr{\B\E[L_M] + \B S \sqrt{MA \log^2(SAK\E[T_M])} + \B S^2A\log^2(SAK\E[T_M])}. \label{eq: pf thm1 tmp 1}
\end{align}
Recall that $L_M \leq  \sqrt{2SAK\log T_M} + SA\log T_M$. Taking expectation from both sides and using Jensen's inequality gets us $\E[L_M] \leq \sqrt{2SAK \log \E[T_M]} + SA \log \E[T_M]$.
Moreover, taking expectation from both sides of \eqref{eq: bound on m}, plugging in the bound on $\E[L_M]$, and concavity of $\log(x)$ implies 
\begin{align*}
M \leq \frac{\E[C_{M}]}{\B} + K + \sqrt{2SAK \log \E[T_M]} + SA \log \E[T_M] + \order\rbr{\frac{\B S^2A}{\cmin}\log \frac{\B KSA}{\cmin}}.
\end{align*}
Substituting this bound in \eqref{eq: pf thm1 tmp 1}, using subadditivity of the square root, and simplifying yields
\begin{align*}
\E[C_M] &\leq K\E[V(\sinit;\theta_*)] + \order\Bigg(\B S \sqrt{KA\log^2 (SAK\E[T_M])} +  S \sqrt{\B \E[C_M]A \log^2 (SAK\E[T_M])}\\
&+ \B S^\frac{5}{4}A^\frac{3}{4}K^\frac{1}{4} \log^\frac{5}{4}(SAK\E[T_M]) + S^2A\sqrt{\frac{\B^3}{\cmin}\log^3 \frac{\B SAK \E[T_M]}{\cmin}}\Bigg).
\end{align*}
Solving for $\E[C_M]$ (by using the primary inequality that $x \leq a \sqrt{x} + b$ implies $x \leq (a + \sqrt b)^2$ for $a, b > 0$), using $K \geq S^2A$, $V(\sinit;\theta_*) \leq \B$, and simplifying the result gives
\begin{align}
\label{eq: pf thm1 tmp2}
&\E[C_M] \leq \Bigg(\order\rbr{S \sqrt{\B A \log^2 (SAK\E[T_M])}} \nonumber \\
&+ \sqrt{K\E[V(\sinit;\theta_*)] + \order\rbr{\B S \sqrt{KA\log^{2.5} (SAK\E[T_M])} + S^2A\sqrt{\frac{\B^3}{\cmin}\log^3 \frac{\B SAK \E[T_M]}{\cmin}}}} \Bigg)^2 \nonumber \\
&\leq \order\rbr{\B S^2A\log^2 \frac{SA\E[T_M]}{\delta}} \nonumber \\
&\quad+ K\E[V(\sinit;\theta_*)] + \order\Bigg(\B S \sqrt{KA\log^{2.5} (SAK\E[T_M])} + S^2A\sqrt{\frac{\B^3}{\cmin}\log^3 \frac{\B SAK \E[T_M]}{\cmin}} \nonumber \\
&\quad+ \B S \sqrt{KA\log^{4} (SAK\E[T_M])} + S^2A\rbr{\frac{{\B}^5}{\cmin}\log^7 \frac{\B SAK\E[T_M]}{\cmin}}^\frac{1}{4}\Bigg) \nonumber \\
&\leq K\E[V(\sinit;\theta_*)] + \order\rbr{\B S \sqrt{KA\log^{4} SAK\E[T_M])} + S^2A \sqrt{\frac{\B^3}{\cmin}\log^4 \frac{\B SAK\E[T_M]}{\cmin}}}.
\end{align}
Note that by simplifying this bound, we can write $\E[C_M] \leq \order\rbr{\sqrt{{\B}^3S^4A^2K^2\E[T_M]/\cmin}}$. On the other hand, we have that $\cmin T_M \leq C_M$ which implies $\E[T_M] \leq \E[C_M]/\cmin$. Isolating $\E[T_M]$ implies $\E[T_M] \leq \order\rbr{{\B}^3S^4A^2K^2/c^3_\text{min}}$. Substituting this bound into \eqref{eq: pf thm1 tmp2} yields
\begin{align*}
&\E[C_M] \leq K\E[V(\sinit;\theta_*)] + \order\rbr{\B S \sqrt{KA\log^{4} \frac{\B SAK}{\cmin}} + S^2A \sqrt{\frac{\B^3}{\cmin}\log^4 \frac{\B SAK}{\cmin}}}.
\end{align*}
We note that this bound holds for any number of $M$ intervals as long as the $K$ episodes have not elapsed. Since, $\cmin > 0$, this implies that the $K$ episodes eventually terminate and the claimed bound of the theorem for $R_K$ holds.
\end{proof}

\subsection{Proof of Theorem~\ref{thm2}}\label{sec:proof:thm2}
\textbf{Theorem} (restatement of Theorem~\ref{thm2})\textbf{.}
Suppose Assumption~\ref{ass: class of ssp} holds. Running the \ssp~algorithm with costs $c_\epsilon(s, a) := \max \{c(s, a), \epsilon\}$ for $\epsilon = (S^2A/K)^{2/3}$ yields
	\begin{align*}
		R_K = \order\rbr{\B S \sqrt{KA}\tilde{L}^2 + (S^2A)^\frac{2}{3}K^\frac{1}{3}(\B^\frac{3}{2}\tilde{L}^2 + \T) + S^2A\T^\frac{3}{2}\tilde{L}^2},
	\end{align*}
	where $\tilde L := \log (K\B\T SA)$ and $\T$ is an upper bound on the expected time the optimal policy takes to reach the goal from any initial state.
\begin{proof}
Denote by $T_K^\epsilon$ the time to complete $K$ episodes if the algorithm runs with the perturbed costs $c_\epsilon(s, a)$ and let $V_\epsilon(\sinit;\theta_*)$, $V^\pi_\epsilon(\sinit;\theta_*)$ be the optimal value function and the value function for policy $\pi$ in the SSP with cost function $c_\epsilon(s, a)$ and transition kernel $\theta_*$. We can write
\begin{align}
R_K&= \E\sbr{\sum_{t=1}^{T_K^\epsilon}c(s_t, a_t) - KV(\sinit;\theta_*)} \nonumber \\
&\leq \E\sbr{\sum_{t=1}^{T_K^\epsilon}c_\epsilon(s_t, a_t) - KV(\sinit;\theta_*)} \nonumber \\
&=  \E\sbr{\sum_{t=1}^{T_K^\epsilon}c_\epsilon(s_t, a_t) - KV_\epsilon(\sinit;\theta_*)} + K\E\sbr{V_\epsilon(\sinit;\theta_*) - V(\sinit;\theta_*)}. \label{eq: thm2 tmp1}
\end{align}
Theorem~\ref{thm1} implies that the first term is bounded by
\begin{align*}
\E\sbr{\sum_{t=1}^{T_K^\epsilon}c_\epsilon(s_t, a_t) - KV_\epsilon(\sinit;\theta_*)} = \order\rbr{\B^\epsilon S \sqrt{KA}L_\epsilon^2 + S^2A \sqrt{\frac{{\B^\epsilon}^3}{\epsilon}}L_\epsilon^2},
\end{align*}
with $L_\epsilon = \log (\B^\epsilon SAK/\epsilon)$ and $\B^\epsilon \leq \B + \epsilon \T$ (to see this note that $V_\epsilon(s;\theta_*) \leq V^{\pi^*}_\epsilon(s;\theta_*) \leq \B + \epsilon \T$).
To bound the second term of \eqref{eq: thm2 tmp1}, we have
\begin{align*}
V_\epsilon(\sinit;\theta_*) \leq V^{\pi^*}_\epsilon(\sinit;\theta_*) \leq V(\sinit;\theta_*) + \epsilon\T.
\end{align*}
Combining these bounds, we can write
\begin{align*}
R_K &= \order\rbr{\B S \sqrt{KA}L_\epsilon^2 + \epsilon\T S \sqrt{KA}L_\epsilon^2 + S^2A \sqrt{\frac{(\B + \epsilon \T)^3}{\epsilon}}L_\epsilon^2 + K \T\epsilon}.
\end{align*}
Substituting $\epsilon = (S^2A/K)^{2/3}$, and simplifying the result with $K \geq S^2A$ and $\B \leq \T$ (since $c(s, a) \leq 1$) implies
\begin{align*}
R_K = \order\rbr{\B S \sqrt{KA}\tilde{L}^2 + (S^2A)^\frac{2}{3}K^\frac{1}{3}(\B^\frac{3}{2}\tilde{L}^2 + \T) + S^2A\T^\frac{3}{2}\tilde{L}^2},
\end{align*}
where $\tilde L = \log (K\B\T SA)$. This completes the proof.
\end{proof}

\end{document}